\documentclass[twoside]{article}
\usepackage[accepted]{aistats2015}
\usepackage{hyperref}
\usepackage{url}
\usepackage{amssymb}
\usepackage{amsmath}
\usepackage{graphicx}
\usepackage{chngcntr}
\usepackage[lined,algonl,boxed,oldcommands]{algorithm2e}
\usepackage{color}
\newcommand{\note}[1]{}
\newcommand{\be}{\begin{eqnarray} \begin{aligned}}
\newcommand{\ee}{\end{aligned} \end{eqnarray} }
\newcommand{\benn}{\begin{eqnarray*} \begin{aligned}}
\newcommand{\eenn}{\end{aligned} \end{eqnarray*} }

\usepackage{subfigure}
\usepackage{lscape}
\usepackage[figuresleft]{rotating}
\usepackage{sidecap}
\usepackage{amsthm}

\makeatletter
\def\thm@space@setup{%
  \thm@preskip=\parskip \thm@postskip=0pt
}
\makeatother

\newtheorem{theorem}{Theorem}[section]
\newtheorem*{theorem*}{Theorem}

\newtheorem{corollary}[theorem]{Corollary}

\newenvironment{definition}[1][Definition]{\begin{trivlist}
\item[\hskip \labelsep {\bfseries #1}]}{\end{trivlist}}

\usepackage{titlesec}
\titlespacing*{\section} {0pt}{1.5ex plus 0ex minus 1ex}{0.5ex plus .2ex minus 0.5ex}
\titlespacing*{\subsection} {0pt}{1ex plus 0ex minus .5ex}{0.5ex plus 0ex minus 0.5ex}
\titlespacing*{\subsubsection}{0pt}{0ex plus 0ex minus .5ex}{0.ex plus 0ex minus 0.5ex}
\titlespacing\paragraph{0pt}{0.ex plus 0.ex minus 0ex}{1em}

\begin{document}
\setlength{\abovedisplayskip}{0.5mm}
\setlength{\belowdisplayskip}{0.5mm}

\runningtitle{Maximally Informative Hierarchical Representations of High-Dimensional Data}

\twocolumn[

\aistatstitle{Maximally Informative Hierarchical Representations \\ of High-Dimensional Data}

\aistatsauthor{ Greg Ver Steeg \And Aram Galstyan}

\aistatsaddress{ Information Sciences Institute \\ University of Southern California \\ gregv@isi.edu 
\And Information Sciences Institute \\ University of Southern California \\ galstyan@isi.edu } ]


\begin{abstract}
We consider a set of probabilistic functions of some input variables as a \emph{representation} of the inputs. 
We present bounds on how informative a representation is about input data. 
We extend these bounds to hierarchical representations so that we can quantify the contribution of each layer towards capturing the information in the original data. 
The special form of these bounds leads to a simple, bottom-up optimization procedure to construct hierarchical representations that are also maximally informative about the data.
This optimization has linear computational complexity and constant sample complexity in the number of variables.  
These results establish a new approach to unsupervised learning of deep representations that is both principled and practical. 
We demonstrate the usefulness of the approach on both synthetic and real-world data. 
\end{abstract}

This paper considers the problem of unsupervised learning of hierarchical representations from high-dimensional data. 
Deep representations are becoming increasingly indispensable for solving the greatest challenges in machine learning including problems in image recognition, speech, and language~\cite{bengioreview}. 
Theoretical foundations have not kept pace, making it difficult to understand why existing methods fail in some domains and succeed in others. 
Here, we start from the abstract point of view that any probabilistic functions of some input variables constitute a representation. The usefulness of a representation depends on (unknown) details about the data distribution. 
Instead of making assumptions about the data-generating process or directly minimizing some reconstruction error, we consider the simple question of how informative a given representation is about the data distribution. 
We give rigorous upper and lower bounds characterizing the informativeness of a representation.
We show that we can efficiently construct representations that optimize these bounds. Moreover, we can add layers to our representation from the bottom up to achieve a series of successively tighter bounds on the information in the data. 
The modular nature of our bounds even allows us to separately quantify the information contributed by each learned latent factor, leading to easier interpretability than competing methods~\cite{richer}. 

Maximizing informativeness of the representation is an objective that is meaningful and well-defined regardless of details about the data-generating process. 
By maximizing an information measure instead of the likelihood of the data under a model, our approach could be compared to lossy compression~\cite{tishby} or coarse-graining~\cite{wolpert2014}. Lossy compression is usually defined in terms of a distortion measure. Instead, we motivate our approach as maximizing the multivariate mutual information (or ``total correlation''~\cite{watanabe}) that is ``explained'' by the representation~\cite{nips2014}. The resulting objective could also be interpreted as using a distortion measure that preserves the most redundant information in a high dimensional system. 
Typically, optimizing over \emph{all probabilistic functions} in a high-dimensional space would be intractable, but the special structure of our objective leads to an elegant set of self-consistent equations that can be solved iteratively.

The theorems we present here establish a foundation to information-theoretically measure the quality of hierarchical representations.
This framework leads to an innovative way to build hierarchical representations with theoretical guarantees in a computationally scalable way. 
Recent results based on the method of Correlation Explanation (CorEx) as a principle for learning~\cite{nips2014} appear as a special case of the framework introduced here. 
CorEx has demonstrated excellent performance for unsupervised learning with data from diverse sources including human behavior, biology, and language~\cite{nips2014} and was able to perfectly reconstruct synthetic latent trees orders of magnitude larger than competing approaches~\cite{tree_survey}.  
After introducing some background in Sec.~\ref{sec:background}, we state the main theorems in Sec.~\ref{sec:bounds} and how to optimize the resulting bounds in Sec.~\ref{sec:optimal}. We show how to construct maximally informative representations in practice in Sec.~\ref{sec:data}. We demonstrate these ideas on synthetic data and real-world financial data in Sec.~\ref{sec:experiments} and conclude in Sec.~\ref{sec:conclusion}. 

\section{Background}\label{sec:background}
Using standard notation~\cite{cover}, capital $X_i$ denotes a random variable taking values in some domain $\mathcal X_i$ and whose instances are denoted in lowercase, $x_i$. We abbreviate multivariate random variables, $X \equiv X_1,\ldots,X_n$, with an associated probability distribution, $p_X(X_1=x_1,\ldots, X_n=x_n)$, which is typically abbreviated to $p(x)$.  We will index different groups of multivariate random variables with superscripts and each multivariate group, $Y^k$, may consist of a different number of variables, $m_k$, with $Y^k \equiv Y^k_1,\ldots, Y^k_{m_k}$(see Fig.~\ref{fig:hierarchy}). The group of groups is written, $Y^{1:r} \equiv Y^1,\ldots, Y^r$.  Latent factors, $Y_j$, will be considered discrete but the domain of the  $X_i$'s is not restricted.

Entropy is defined in the usual way as $H(X) \equiv \mathbb E_X [ \log 1/p(x)]$. We use natural logarithms so that the unit of information is nats. 
Higher-order entropies can be constructed in various ways from this standard definition. For instance, the mutual information between two random variables, $X_1$ and $X_2$ can be written $I(X_1:X_2) = H(X_1) + H(X_2) - H(X_1,X_2)$. Mutual information can also be seen as the reduction of uncertainty in one variable, given information about the other, $ I(X_1:X_2) = H(X_1) - H(X_1|X_2)$. 

The following measure of mutual information among many variables was first introduced as ``total correlation''~\cite{watanabe} and is also called multi-information~\cite{multiinformation} or multivariate mutual information~\cite{kraskov_cluster}. 
\be\label{eq:tc}
TC(X) &\equiv \sum_{i=1}^n H(X_i) - H(X)  \\
&= D_{KL}\left(p(x) || \prod_{i=1}^n p(x_i)\right)
\ee
Clearly, $TC(X)$ is non-negative since it can be written as a KL divergence. For $n=2$, $TC(X)$ corresponds to the mutual information, $I(X_1:X_2)$. While we use the original terminology of ``total correlation'', in modern terms it would be better described as a measure of total dependence. $TC(X)$ is zero if and only if all the $X_i$'s are independent. 

The total correlation among a group of variables, $X$, after conditioning on some other variable, $Y$, can be defined in a straightforward way. 
\benn
TC(X|Y) &\equiv \sum_{i} H(X_i|Y) - H(X|Y) \\
&= D_{KL}\left(p(x|y) || \prod_{i=1}^n p(x_i|y)\right)
\eenn
We can measure the extent to which $Y$ (approximately) {\it explains} the correlations in $X$ by looking at how much the total correlation is reduced, 
\be\label{eq:tcxy}
TC(X;Y) &\equiv TC(X) - TC(X|Y) \\
&= \sum_{i=1}^n I(X_i:Y) - I(X:Y).
\ee
We use semicolons as a reminder that $TC(X;Y)$ is not symmetric in the arguments, unlike mutual information. $TC(X|Y)$ is zero (and $TC(X;Y)$ maximized) if and only if the distribution of $X$'s conditioned on $Y$ factorizes. This would be the case if $Y$ contained full information about all the common causes among $X_i$'s in which case we recover the standard statement, an exact version of the one we made above, that $Y$ {\it explains} all the correlation in $X$.
  $TC(X|Y)=0$ can also be seen as encoding conditional independence relations and is therefore relevant for constructing graphical models~\cite{pearl}. 
  This quantity has appeared as a measure of the {\it redundant} information that the $X_i$'s carry about $Y$~\cite{synergy} and this interpretation has been explored in depth~\cite{williamsbeer,griffith}. 

\section{Representations and Information}\label{sec:bounds}

\begin{figure}[tbp] 
   \centering
   \includegraphics[width=0.65\columnwidth]{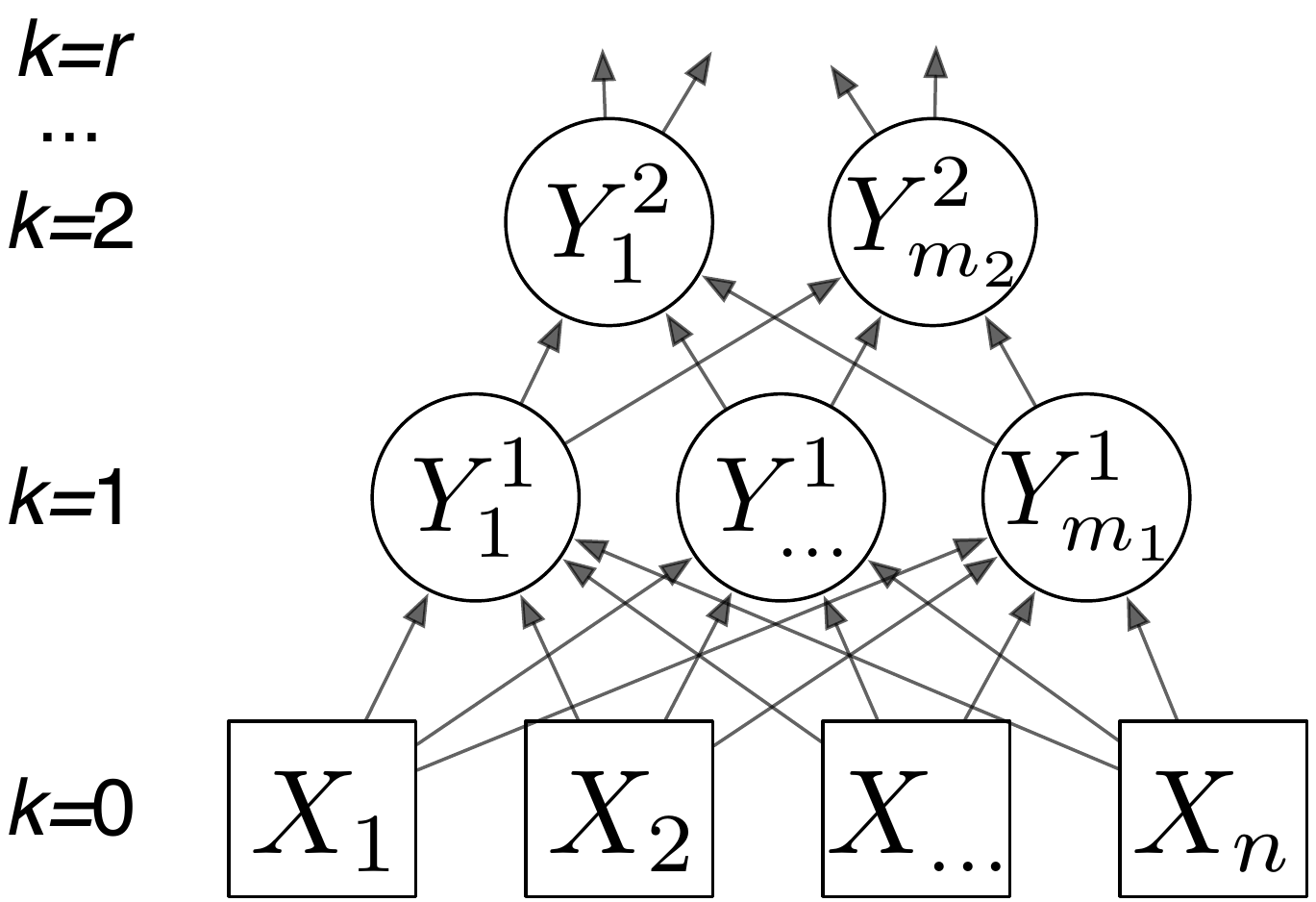} 
   \caption{In this graphical model, the variables on the bottom layer ($X_i$'s) represent observed variables. The variables in each subsequent layer represent coarse-grained descriptions that explain the correlations in the layer below. Thm.~\ref{hierarchy} quantifies how each layer contributes to successively tighter bounds on $TC(X)$.}
   \label{fig:hierarchy}
\end{figure}

\begin{definition}
The random variables $Y \equiv Y_1,\ldots,Y_m$ constitute a \emph{representation} of $X$ if the joint distribution factorizes, $p(x,y) = \prod_{j=1}^m p(y_j|x) p(x), \forall x \in \mathcal X, \forall j \in \{1,\ldots,m\}, \forall y_j \in \mathcal Y_j$. A representation is completely defined by the domains of the variables and the conditional probability tables, $p(y_j|x)$.
\end{definition}
\begin{definition}
The random variables $Y^{1:r} \equiv Y^1,\ldots,Y^r$ consitute a \emph{hierarchical representation} of $X$ if $Y^1$ is a representation of $X$ and $Y^k$ is a representation of $Y^{k-1}$ for $k=2,\ldots,r$. (See Fig.~\ref{fig:hierarchy}.) 
\end{definition}
We will be particularly interested in bounds quantifying how informative $Y^{1:r}$ is about $X$. These bounds will be used to search for representations that are maximally informative. 
These definitions of representations are quite general and include (two-layer) RBMs and deterministic representations like auto-encoders as a special case. Note that this definition only describes a prescription for generating coarse-grained variables ($Y$'s) and does not specify a generative model for $X$.

\begin{theorem}
\emph{Basic Decomposition of Information}
\label{basic}
If $Y$ is a representation of $X$ and we define,
\begin{align}\label{eq:tcl}
TC_{L} (X; Y) &\equiv  \sum_{i=1}^n I(Y:X_i) - \sum_{j=1}^m I(Y_j:X),
\end{align}
then the following bound and decomposition holds. 
\be\label{eq:basic}
TC(X) \geq TC(X;Y) = TC(Y) + TC_L(X;Y)
\ee
\end{theorem}
\begin{proof} A proof is provided in Sec.~\ref{sec:lower}.\end{proof}

\begin{corollary}
\label{tclb}
$TC(X;Y)  \geq TC_{L} (X; Y)$
\end{corollary}
This follows from Eq.~\ref{eq:basic} due to the non-negativity of total correlation. Note that this lower bound is zero if $Y$ contains no information about $X$, i.e., if $\forall x\in \mathcal X, p(y_j|x) = p(y_j)$.

\begin{theorem}
\emph{Hierarchical Lower Bound on $TC(X)$}
\label{hierarchy}
If $Y^{1:r}$ is a hierarchical representation of $X$ and we define $Y^0 \equiv X$,
\be
TC(X) \geq \sum_{k=1}^r TC_{L} (Y^{k-1}; Y^k).
\ee 
\end{theorem}
\begin{proof}
This follows from writing down Thm.~\ref{basic}, $TC(X) \geq TC(Y^1) + TC_{L} (X;Y^1)$. Next we repeatedly invoke Thm.~\ref{basic} to replace $TC(Y^{k-1})$ with its lower bound in terms of $Y^k$. The final term, $TC(Y^r)$, is non-negative and can be discarded. Alternately, if $m_r$ is small enough it could be estimated directly or if $m_r=1$ this implies $TC(Y^r)=0$. 
\end{proof}
\begin{theorem}
\emph{Upper Bounds on $TC(X)$}
\label{tcub}

If $Y^{1:r}$ is a hierarchical representation of $X$ and we define $Y^0 \equiv X$, and additionally $m_r=1$ and all variables are discrete, then, 
\benn\label{eq:tcub}
TC(X) &\leq  \sum_{k=1}^r \left( TC_L(Y^{k-1};Y^k) + \sum_{i=1}^{m_{k-1}} H(Y^{k-1}_i|Y^k) \right).
\eenn
\end{theorem}
\begin{proof} A proof is provided in Sec.~\ref{sec:upper}.\end{proof}
The reason for stating this upper bound is to show it is equal to the lower bound plus the term $\sum_k \sum_{i=1}^{m_{k-1}} H(Y^{k-1}_i|Y^k)$. If each variable is perfectly predictable from the layer above it we have a guarantee that our bounds are tight and our representation provably contains all the information in $X$. 
Thm.~\ref{tcub} is stated for discrete variables for simplicity but a similar bound holds if $X_i$ are not discrete.

\paragraph{Bounds on $H(X)$}
We focus above on total correlation as a measure of information. One intuition for this choice is that uncorrelated subspaces are, in a sense, not truly high-dimensional and can be characterized separately. On the other hand, the entropy of $X$, $H(X)$, can naively be considered the appropriate measure of the ``information.''\footnote{The InfoMax principle~\cite{linsker} constructs representations to directly maximize the (instinctive but misleading quantity~\cite{icml2014}) mutual information, $I(X:Y)$. Because InfoMax ignores the multivariate structure of the input space, it cannot take advantage of our hierarchical decompositions. The efficiency of our method and the ability to progressively bound information rely on this decomposition.}
Estimating the multivariate mutual information is really the hard part of estimating $H(X)$. We can write $H(X) = \sum_{i=1}^n H(X_i) - TC(X)$. The marginal entropies, $H(X_i)$ are typically easy to estimate so that our bounds on $TC(X)$ in Thm.~\ref{hierarchy} and Thm.~\ref{tcub} directly translate into bounds on the entropy as well.  

The theorems above provide useful bounds on information in high-dimensional systems for three reasons. First, they show how to additively decompose information. Second, in Sec.~\ref{sec:optimal} we show that $TC_{L}(Y^{k-1};Y^k)$ can be efficiently optimized over, leading to progressively tighter bounds. Finally, $TC_{L}(Y^{k-1};Y^k)$ can be efficiently estimated even using small amounts of data, as described in Sec.~\ref{sec:data}. 

\section{Optimized Representations}\label{sec:optimal}

Thm.~\ref{hierarchy} suggests a way to build optimally informative hierarchical representation from the bottom up. Each layer can be optimized to maximally explain the correlations in the layer below. The contributions from each layer can be simply summed to provide a progressively tighter lower bound on the total correlation in the data itself. 
\be
\label{eq:basicopt}
\max_{\forall j, p(y_j^1|x)} TC_{L}(X;Y^1)
\ee
After performing this optimization, in principle one can continue to maximize $TC_{L}(Y^1;Y^2)$ and so forth up the hierarchy. 
As a bonus, representations with different numbers of layers and different numbers of variables in each layer can be easily and objectively compared according to the tightness of the lower bound on $TC(X)$ that they provide using Thm.~\ref{hierarchy}.

While solving this optimization and obtaining accompanying bounds on the information in $X$ would be convenient, it does not appear practical because the optimization is over all possible probabilistic functions of $X$. We now demonstrate the surprising fact that the solution to this optimization implies a solution for $p(y_j|x)$ with a special form in terms of a set of self-consistent equations that can be solved iteratively.

\subsection{A Single Latent Factor}\label{sec:single}

First, we consider the simple representation for which $Y^1 \equiv Y^1_1$ consists of a single random variable taking values in some discrete space. In this special case, $TC(X;Y^1) = TC_L(X;Y^1) = \sum_i I(Y_1^1:X_i) - I(Y_1^1:X)$. Optimizing Eq.~\ref{eq:basicopt} in this case leads to 
\be
\label{eq:m1}
\max_{p(y_1|x)} \sum_{i=1}^n I(Y_1^1:X_i) - I(Y_1^1:X).
\ee

Instead of looking for the optimum of this expression, we consider the optimum of a slightly more general expression whose solution we will be able to re-use later. Below, we omit the superscripts and subscripts on $Y$ for readability. 
Define the ``$\alpha$-Ancestral Information'' that $Y$ contains about $X$ as follows, $AI_\alpha (X;Y) \equiv \sum_{i=1}^n \alpha_i I(Y:X_i) - I(Y:X)$, where $\alpha_i \in [0,1]$. The name is motivated by results that show that if $AI_{\alpha}(X;Y)$ is positive for some $\alpha$, it implies the existence of common ancestors for some ($\alpha$-dependent) set of $X_i$'s in any DAG that describes $X$~\cite{steudelay}. We do not make use of those results, but the overlap in expressions is suggestive. 
We consider optimizing the ancestral information where $\alpha_i \in [0,1]$ keeping in mind that the special case of $\forall i,\alpha_i=1$ reproduces Eq.~\ref{eq:m1}. 
\be
\label{eq:AI} 
\max_{p(y|x)} \sum_{i=1}^n \alpha_i I(Y:X_i) - I(Y:X)
\ee
We use Lagrangian optimization (detailed derivation is in Sec.~\ref{sec:derive}) to find the solution.
\begin{align}
p(y|x) &= \frac1{Z(x)} p(y) \prod_{i=1}^n \left(\frac{p(y|x_i)}{p(y)}\right)^{\alpha_{i}} \label{eq:label}
\end{align}
Normalization is guaranteed by $Z(x)$. 
While Eq.~\ref{eq:label} appears as a formal solution to the problem, we must remember that it is defined in terms of quantities that themselves depend on $p(y|x)$.
\be\label{eq:marg}
p(y|x_i) &= \sum_{\bar x \in \mathcal X} p(y|\bar x) p(\bar x) \delta_{\bar x_i,x_i}/p(x_i) \\
 p(y) &= \sum_{\bar x \in \mathcal X} p(y|\bar x) p(\bar x)  
\ee
Eq.~\ref{eq:marg} simply states that the marginals are consistent with the labels $p(y|x)$ for a given distribution, $p(x)$.

This solution has a remarkable property. Although our optimization problem was over all possible probabilistic functions, $p(y|x)$, Eq.~\ref{eq:label} says that this function can be written in terms of a linear (in $n$, the number of variables) number of parameters which are just marginals involving the hidden variable $Y$ and each $X_i$. We show how to exploit this fact to solve optimization problems in practice using limited data in Sec.~\ref{sec:data}. 

\subsection{Iterative Solution}\label{sec:update}
The basic idea is to iterate between the self-consistent equations to converge on a fixed-point solution. 
Imagine that we start with a particular representation at time $t$, $p^t(y|x)$ (ignoring the difficulty of this for now). Then, we estimate the marginals, $p^{t}(y|x_i), p^{t}(y)$ using Eq.~\ref{eq:marg}. Next, we update $p^{t+1}(y|x)$ according to the rule implied by Eq.~\ref{eq:label},
\be\label{eq:update}
p^{t+1}(y|x) &= \frac1{Z^{t+1}(x)} p^{t}(y) \prod_{i=1}^n \left(\frac{p^{t}(y|x_i)}{p^{t}(y)}\right)^{\alpha_{i}}.
\ee
Note that $Z^{t+1}(x)$ is a partition function that can be easily calculated for each $x$ (by summing over the latent factor, $Y$, which is typically taken to be binary). 
$$ Z^{t+1}(x) = \sum_{y \in \mathcal Y} p^{t}(y) \prod_{i=1}^n \left(\frac{p^{t}(y|x_i)}{p^{t}(y)}\right)^{\alpha_{i}}$$
\begin{theorem}
\label{convergence}
Assuming $\alpha_1,\ldots,\alpha_n \in [0,1]$, iterating over the update equations given by Eq.~\ref{eq:update} and Eq.~\ref{eq:marg} never decreases the value of the objective in Eq.~\ref{eq:AI} and is guaranteed to converge to a stationary fixed point.
\end{theorem}
Proof is provided in Sec.~\ref{sec:Z}.

At this point, notice a surprising fact about this partition function. Rearranging Eq.~\ref{eq:label} and taking the $\log$ and expectation value,
\be\label{eq:Z}
\mathbb E \left[ \log Z(x) \right] &= \mathbb E \left[ \log \frac{p(y)}{p(y|x)} \prod_{i=1}^n \left(\frac{p(y|x_i)}{p(y)}\right)^{\alpha_{i}} \right]  \\ &= \sum_{i=1}^n \alpha_i I(Y:X_i) - I(Y:X) 
\ee
The expected log partition function (sometimes called the \emph{free energy}) is just the value of the objective we are optimizing. We can estimate it at each time step and it will converge as we approach the fixed point. 

\subsection{Multiple Latent Factors}

Directly maximizing $TC_{L}(X;Y)$, which in turn bounds $TC(X)$, with $m>1$ is intractable for large $m$. Instead we construct a lower bound that shares the form of Eq.~\ref{eq:AI} and therefore is tractable.  
\be\label{eq:tclb2}
\lefteqn{TC_{L}(X;Y) \equiv \sum_{i=1}^n I(Y:X_i) - \sum_{j=1}^m I(Y_j:X)} \\
&=& \sum_{i=1}^n \sum_{j=1}^m I(Y_j : X_i |Y_{1:j-1}) - \sum_{j=1}^m I(Y_j:X) \\
& \geq&  \sum_{j=1}^m \left(\sum_{i=1}^n  \alpha_{i,j} I(Y_j : X_i) -  I(Y_j:X) \right)
\ee
In the second line, we used the chain rule for mutual information~\cite{cover}. Note that in principle the chain rule can be applied for any ordering of the $Y_j$'s.  
In the final line, we rearranged summations to highlight the decomposition as a sum of terms for each hidden unit, $j$. Then, we simply \emph{define} $\alpha_{i,j}$ so that, 
\be\label{eq:alpha}
 I(Y_j : X_i |Y_{1:j-1}) \geq  \alpha_{i,j} I(Y_j : X_i).
\ee
An intuitive way to interpret $ I(Y_j : X_i |Y_{1:j-1})/I(Y_j : X_i)  $ is as the fraction of $Y_j$'s information about $X_i$ that is unique (i.e. not already present in $Y_{1:j-1}$). Cor.~\ref{tclb} implies that $\alpha_{i,j} \leq 1$ and it is also clearly non-negative. 

Now, instead of maximizing $TC_{L}(X;Y)$ over all hidden units, $Y_j$, we maximize this lower bound over both $p(y_j|x)$ and $\alpha$, subject to some constraint, $c_{i,j} (\alpha_{i,j}) = 0$ that guarantees that $\alpha$ obeys Eq.~\ref{eq:alpha}.  
\be\label{eq:fullopt}
&\max_{ \substack{\alpha_{i,j},p(y_j|x) \\ c_{i,j}(\alpha_{i,j})=0} }
~~\sum_{j=1}^m \left(\sum_{i=1}^n  \alpha_{i,j} I(Y_j : X_i) -  I(Y_j:X) \right)
\ee
We solve this optimization problem iteratively, re-using our previous results.
First, we fix $\alpha$ so that this optimization is equivalent to solving $j$ problems of the form in Eq.~\ref{eq:AI} in parallel by adding indices to our previous solution, 
\be
p(y_j|x) &= \frac1{Z_j(x)} p(y_j) \prod_{i=1}^n \left(\frac{p(y_j|x_i)}{p(y_j)}\right)^{\alpha_{i,j}}. \label{eq:label_m}
\ee 
The results in Sec.~\ref{sec:update} define an incremental update scheme that is guaranteed to increase the value of the objective.
Next, we fix $p(y_j |x)$ and update $\alpha$ so that it obeys Eq.~\ref{eq:alpha}. 
Updating $p(y_j|x)$ never decreases the objective and as long as $\alpha_{i,j} \in [0,1]$, the total value of the objective is upper bounded.
Unfortunately, the $\alpha$-update scheme is not guaranteed to increase the objective. Therefore, we stop iterating if changes in $\alpha$ have not increased the objective over a time window including the past ten iterations. In practice we find that convergence is obtained quickly with few iterations as shown in Sec.~\ref{sec:experiments}. Specific choices for updating $\alpha$ are discussed next.

\paragraph{Optimizing the Structure}
Looking at Eq.~\ref{eq:label_m}, we see that $\alpha_{i,j}$ really defines the input variables, $X_i$, that $Y_j$ depends on. If $\alpha_{i,j} = 0$, then $Y_j$ is independent of $X_i$ conditioned on the remaining $X$'s. Therefore, we say that $\alpha$ defines the structure of the representation. 
For $\alpha$ to satisfy the inequality in the last line of Eq.~\ref{eq:tclb2}, we can use the fact that $\forall j, I(Y_j:X_i) \leq I(Y:X_i)$. Therefore, we can lower bound $I(Y:X_i)$ using any convex combination of  $I(Y_j:X_i)$ by demanding that $\forall i, \sum_j \alpha_{i,j} = 1$. 
A particular choice is as follows:
\be\label{eq:tree}
\alpha_{i,j} = \mathbb I [ j = \arg \max_{\bar j} I(X_i : Y_{\bar j})].\vspace{-0mm}
\ee
This leads to a tree structure in which each $X_i$ is connected to only a single (most informative) hidden unit in the next layer. This strategy reproduces the latent tree learning method previously introduced~\cite{nips2014}.  

Based on Eq.~\ref{eq:alpha}, we propose a heuristic method to estimate $\alpha$ that does not restrict solutions to trees. 
For each data sample $l=1,\ldots,N$ and variable $X_i$, we check if $X_i$ correctly predicts $Y_j$ (by counting $d_{i,j}^l \equiv  \mathbb I[ \arg \max_{y_j} \log p(Y_j=y_j|x^{(l)}) = \arg \max_{y_j} \log p(Y_j=y_j|x^{(l)}_i)/p(Y_j=y_j)]$.
For each $i$, we sort the $j$'s according to which ones have the most correct predictions (summing over $l$).  
Then we set $\alpha_{i,j}$ as the percentage of samples for which $d_{i,j}^l =1$ while $d_{i,1}^l=\cdots = d_{i,j-1}^l =0$ . How well this approximates the fraction of unique information in Eq.~\ref{eq:alpha} has not been determined, but empirically it gives good results. Choosing the best way for efficiently lower-bounding the fraction of unique information is a question for further research. 


\section{Complexity and Implementation}\label{sec:data}

Multivariate measures of information have been used to capture diverse concepts such as redundancy, synergy, ancestral information, common information, and complexity. Interest in these quantities remains somewhat academic since they typically cannot be estimated from data except for toy problems. Consider a simple problem in which $X_1,\ldots,X_n$ represent $n$ binary random variables. The size of the state space for $X$ is $2^n$. The information-theoretic quantities we are interested in are functionals of the full probability distribution. Even for relatively small problems with a few hundred variables, the number of samples required to estimate the full distribution is impossibly large. 

Imagine that we are given $N$ iid samples, $x^{(1)},\ldots,x^{(l)},\ldots,x^{(N)}$, from the unknown distribution $p(x)$. 
A naive estimate of the probability distribution is given by $\hat p(x) \equiv \frac1N \sum_{l=1}^N \mathbb I[x=x^{(l)}]$. Since $N$ is typically much smaller than the size of the state space, $N \ll 2^n$, this would seem to be a terrible estimate. On the other hand, if we are just estimating a marginal like $p(x_i)$, then a simple Chernoff bound guarantees that our estimation error decreases exponentially with $N$. 

Our optimization seemed intractable because it is defined over $p(y_j|x)$. If we approximate the data distribution with $\hat p(x)$, then instead of specifying $p(y_j | x)$ for all possible values of $x$, we can just specify $p(y_j | x^{(l)})$ for the $l=1,\ldots, N$ samples that have been seen. The next step in optimizing our objective (Sec.~\ref{sec:update}) is to calculate the marginals $p(y_j|x_i)$. To calculate these marginals with fixed error only requires a constant number of samples (constant w.r.t. the number of variables). Finally, updating the labels, $p(y_j | x^{(l)})$, amounts to calculating a log-linear function of the marginals (Eq.~\ref{eq:label_m}). 

Similarly, $\log Z_j(x^{(l)})$, is just a random variable that can be calculated easily for each sample and the sample mean provides an estimate of the true mean. But we saw in Eq.~\ref{eq:Z} that the average of this quantity is just (the $j$-th component of) the objective we are optimizing. 
This allows us to estimate successively tighter bounds for $TC(X;Y)$ and $TC(X)$ for very high-dimensional data. In particular, we have,
\be\label{eq:est}
TC(X) \geq TC_L(X;Y) \approx \sum_{j=1}^m \frac1N \sum_{l=1}^N \log Z_j (x^{(l)}).
\ee

\paragraph{Algorithm and computational complexity} 
The pseudo-code of the algorithm is laid out in detail in \cite{nips2014} with the procedure to update $\alpha$ altered according to the previous discussion. Consider a dataset with $n$ variables and $N$ samples for which we want to learn a representation with $m$ latent factors. At each iteration, we have to update the marginals $p(y_j | x_i), p(y_j)$, the structure $\alpha_{i,j}$, and re-calculate the labels for each sample, $p(y_j | x^{(l)})$. These steps each require $O(m\cdot n \cdot N)$ calculations. 
Note that instead of using $N$ samples, we could use a mini-batch of fixed size at each update to obtain fixed error estimates of the marginals. 
We can stop iterating after convergence or some fixed number of iterations. Typically a very small number of iterations suffices, see Sec.~\ref{sec:experiments}.

\paragraph{Hierarchical Representation} 
To build the next layer of the representation, we need samples from $p_{Y^1}(y^1)$. In practice, for each sample, $x^{(l)}$, we construct the maximum likelihood label for each $y^1_j$ from $p(y_j^1 | x^{(l)})$, the solution to Eq.~\ref{eq:fullopt}. Empirically, most learned representations are nearly deterministic so this approximation is quite good.

\paragraph{Quantifying contributions of hidden factors} The benefit of adding layers of representations is clearly quantified by Thm.~\ref{hierarchy}. If the contribution at layer $k$ is smaller than some threshold (indicating that the total correlation among variables at layer $k$ is small) we can set a stopping criteria. Intuitively, this means that we stop learning once we have a set of nearly independent factors that explain correlations in the data. Thus, a criteria similar to independent component analysis (ICA)~\cite{ica} appears as a byproduct of correlation explanation. 
Similarly, the contribution to the objective for each latent factor, $j$, is quantified by $\sum_{i} \alpha_{i,j} I(Y_j:X_i) - I(Y_j:X) = \mathbb E[\log Z_j(x)]$. Adding more latent factors beyond a certain point leads to diminishing returns. This measure also allows us to do component analysis, ranking the most informative signals in the data.

\paragraph{Continuous-valued data} The update equations in Eq.~\ref{eq:update} depend on ratios of the form $p(y_j|x_i) / p(y_j)$. For discrete data, this can be estimated directly. For continuous data, we can use Bayes' rule to write this as $p(x_i | y_j) / p(x_i)$. Next, we parametrize each marginal so that $X_i  | {Y_j=k}  \sim \mathcal N(\mu_{i,j,k} , \sigma_{i,j,k})$. Now to estimate these ratios, we first estimate the parameters (this is easy to do from samples) and then calculate the ratio using the parametric formula for the distributions. Alternately, we could estimate these density ratios non-parametrically~\cite{icml2014} or using other prior knowledge.

\section{Experiments}\label{sec:experiments}

We now present experiments constructing hierarchical representations from data by optimizing Eq.~\ref{eq:fullopt}. 
The only change necessary to implement this optimization using available code and pseudo-code~\cite{nips2014,corex_code} is to alter the $\alpha$-update rule according to the discussion in the previous section.
We consider experiments on synthetic and real-world data. We take the domain of latent factors to be binary and we must also specify the number of hidden units in each layer. 

\paragraph{Synthetic data} 
The special case where $\alpha$ is set according to Eq.~\ref{eq:tree} creates tree-like representations recovering the method of previous work~\cite{nips2014}. That paper demonstrated the ability to perfectly reconstruct synthetic latent trees in time $O(n)$ while state-of-the-art techniques are at least $O(n^2)$~\cite{tree_survey}. It was also shown that for high-dimensional, highly correlated data, CorEx outperformed all competitors on a clustering task including ICA, NMF, RBMs, k-means, and spectral clustering. Here we focus on synthetic tests that gauge our ability to measure the information in high-dimensional data and to show that we can do this for data generated according to non-tree-like models. 

To start with a simple example,
imagine that we have four independent Bernoulli variables, $Z_0, Z_1, Z_2, Z_3$ taking values $0,1$ with probability one half. 
Now for $j=0,1,2,3$ we define random variables
$X_i \sim \mathcal N(Z_j, 0.1)$, for $i=100j+1,\ldots, 100j+100$.
We draw 100 samples from this distribution and then shuffle the columns. The raw data is shown in Fig.~\ref{fig:single}(a), along with the data columns and rows sorted according to learned structure, $\alpha$, and the learned factors, $Y_j$, which perfectly recovers structure and $Z_j$'s for this simple case (Fig.~\ref{fig:single}(b)). More interestingly, we see in Fig.~\ref{fig:single}(c) that only three iterations are required for our lower bound (Eq.~\ref{eq:est}) to come within a percent of the true value of $TC(X)$. This provides a useful signal for learning: increasing the size of the representation by increasing the number of hidden factors or the size of the state space of $Y$ cannot increase the lower bound because it is already saturated. 

\begin{figure}[htbp] 
   \centering
   (a)
   \includegraphics[width=0.8\columnwidth]{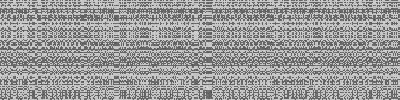} \\ \vspace{2mm}
   
   (b)      
   \includegraphics[width=0.8\columnwidth]{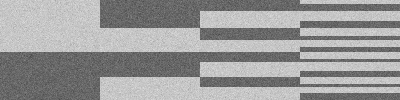} \\ \vspace{2mm}
   
   (c)
         \includegraphics[width=0.8\columnwidth]{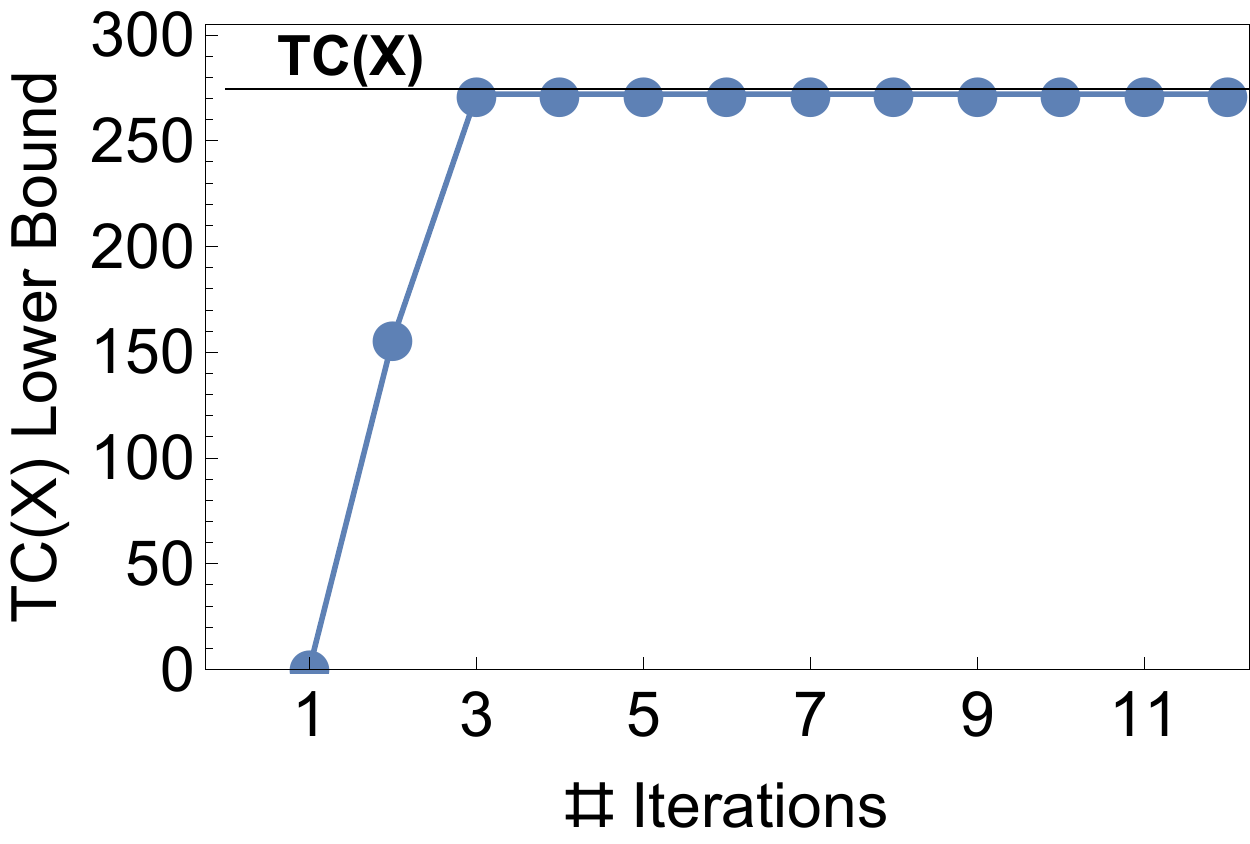}  
   \caption{(a) Randomly generated data with permuted variables. (b) Data with columns and rows sorted according to $\alpha$ and $Y_j$ values. (c) Starting with a random representation, we show the lower bound on total correlation at each iteration. It comes within a percent of the true value after only three iterations. }
   \label{fig:single}
\end{figure}

For the next example, we repeat the same setup except $Z_3 = Z_0 + Z_1$. If we learn a representation with three binary latent factors, then variables in the group $X_{301},\ldots, X_{400}$ should belong in overlapping clusters. 
Again we take 100 samples from this distribution.
For this example, there is no analytic form to estimate $TC(X)$ but we see in Fig.~\ref{fig:multiple}(a) that we quickly converge on a lower bound (Sec.~\ref{sec:restart} shows similar convergence for real-world data).  Looking at Eq.~\ref{eq:label_m}, we see that $Y_j$ is a function of $X_i$ if and only if $\alpha_{i,j} >0$. Fig.~\ref{fig:multiple}(b) shows that $Y_1$ alone is a function of the first 100 variables, etc., and that $Y_1$ and $Y_2$ both depend on the last group of variables, while $Y_3$ does not. In other words, the overlapping structure is correctly learned and we still get fast convergence in this case. 
When we increase the size of the synthetic problems, we get the same results and empirically observe the expected linear scaling in computational complexity.\footnote{With an unoptimized implementation, it takes about 12 minutes to run this experiment with 20,000 variables on a 2012 Macbook Pro.}

\begin{figure}[htbp] 
   \centering
%
%
(a)
         \includegraphics[width=0.8\columnwidth]{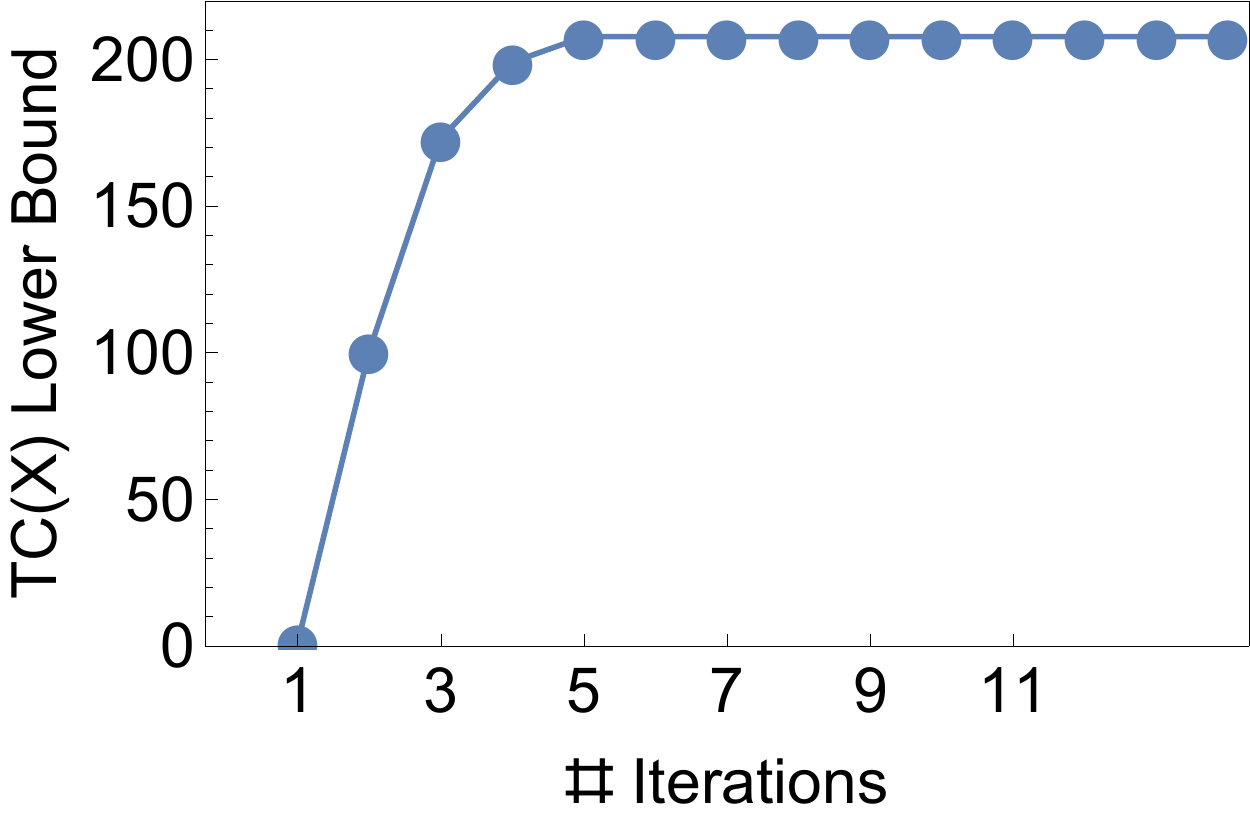} \\ \vspace{2mm}

 (b)
  \includegraphics[width=0.93\columnwidth]{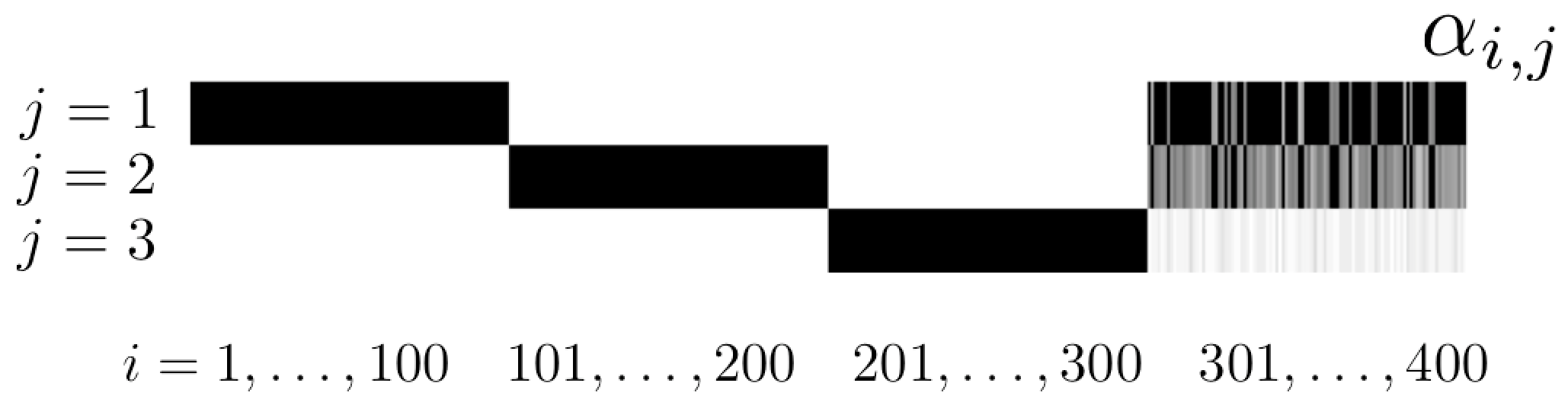} 
   \caption{ (a) Convergence rates for the overlapping clusters example. (b) Adjacency matrix representing $\alpha_{i,j}$. CorEx correctly clusters variables including overlapping clusters.}
   \label{fig:multiple}
   \vspace{-3mm}
\end{figure}

\paragraph{Finance data} For a real-world example, we consider financial time series. We took the monthly returns for companies on the S\&P 500 from 1998-2013\footnote{Data is freely available at \url{www.quantquote.com}.}. We included only the 388 companies which were on the S\&P 500 for the entire period. We treated each month's returns as an iid sample (a naive approximation~\cite{econometric}) from this 388 dimensional space. We use a representation with $m_1 =20, m_2=3, m_3=1$ and $Y_j$ were discrete trinary variables. 

Fig.~\ref{fig:spgraph} shows the overall structure of the learned hierarchical model. Edge thickness is determined by $\alpha_{i,j} I(X_i:Y_j)$. We thresholded edges with weight less than 0.16 for legibility. The size of each node is proportional to the total correlation that a latent factor explains about its children, as estimated using $\mathbb E(\log Z_j(x))$. Stock tickers are color coded according to the Global Industry Classification Standard (GICS) sector. Clearly, the discovered structure captures several significant sector-wide relationships. A larger version is shown in Fig.~\ref{fig:big}. For comparison, in Fig.~\ref{fig:nnet} we construct a similar graph using restricted Boltzmann machines. No useful structure is apparent. 

\begin{figure*}[tbp] 
   \centering
   \includegraphics[width=7in]{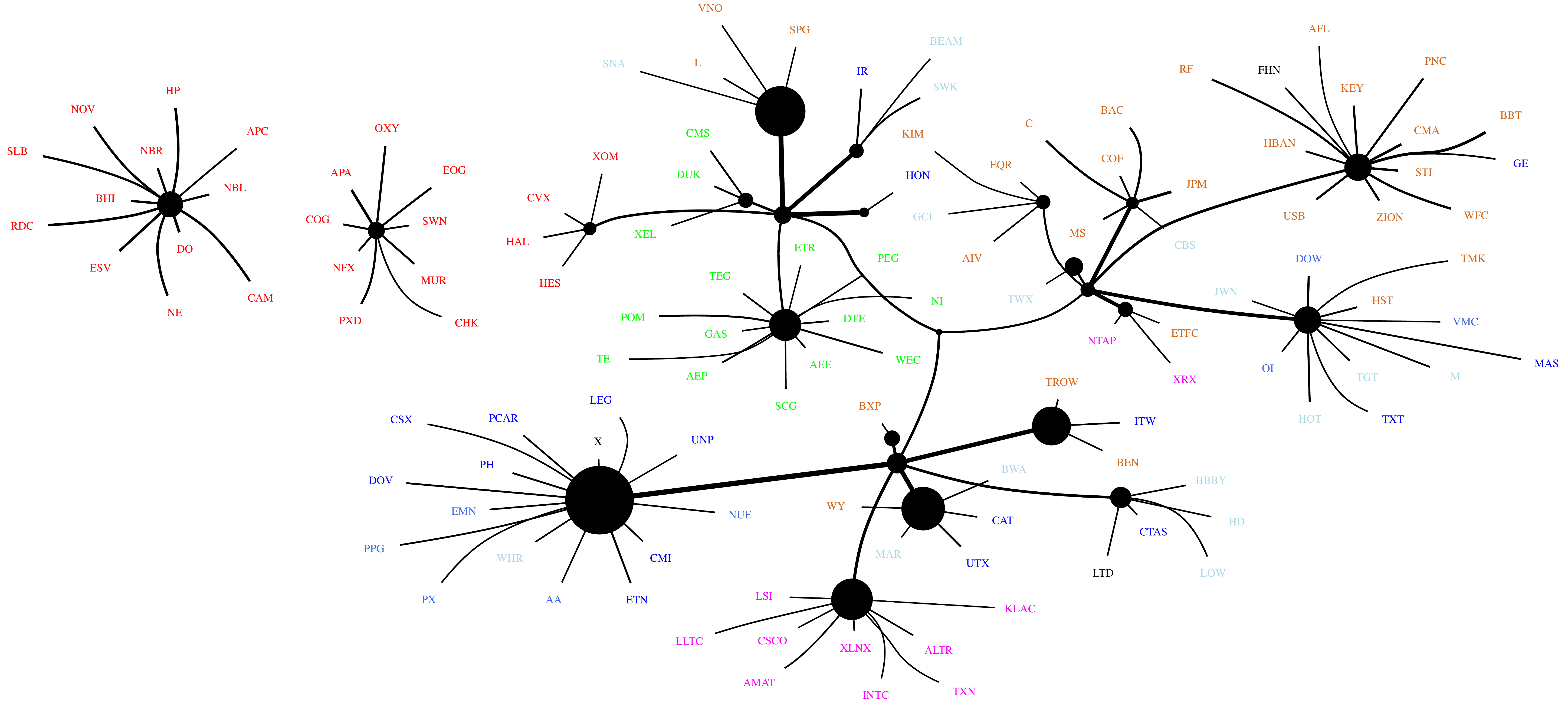} \\
   \includegraphics[width=6.5in]{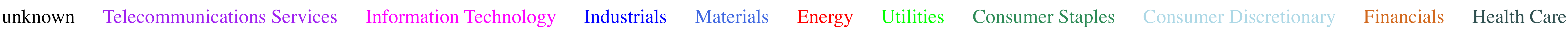} \vspace{-3mm}
   \caption{A thresholded graph showing the overall structure of the representation learned from monthly returns of S\&P 500 companies. Stock tickers are colored (online) according to their GICS sector. Edge thickness is proportional to mutual information and node size represents multivariate mutual information among children.} \vspace{-3mm}
   \label{fig:spgraph}
\end{figure*}

 We interpret $\widehat {TC_L}(X=x^{(l)};Y) \equiv \sum_j \log Z_j(x^{(l)})$ as the point-wise total correlation because its mean over all samples is our estimate of $TC_L(X;Y)$ (Eq.~\ref{eq:est}). We interpret deviation from the mean as a kind of surprise.  
In Fig.~\ref{fig:finance}, we compare the time series of the S\&P 500 to this point-wise total correlation. 
This measure of anomalous behavior captures the market crash in 2008 as the most unusual event of the decade. 

\begin{figure}[h]
   \centering
   \includegraphics[width=0.8\columnwidth]{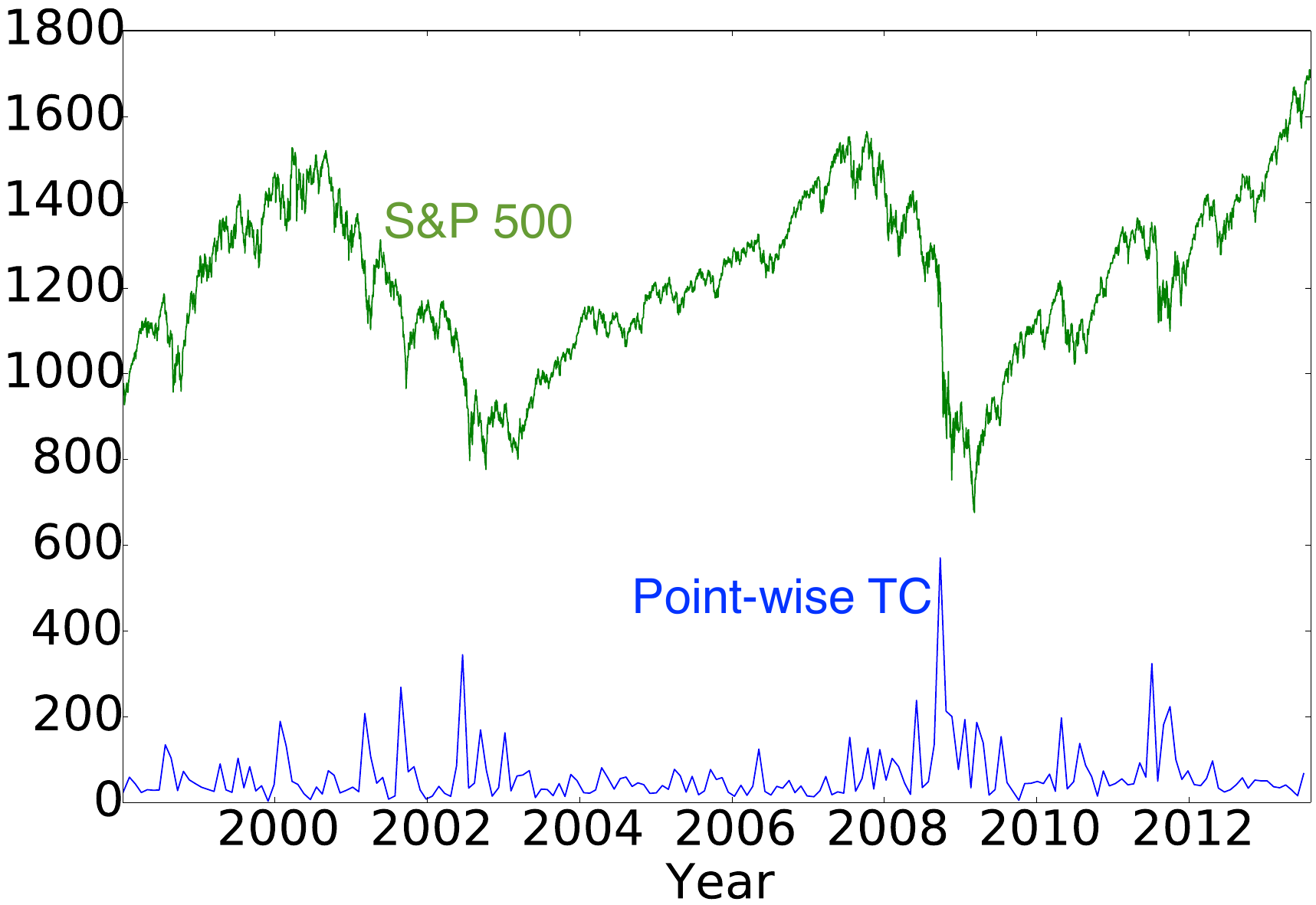} 
\caption{The S\&P 500 over time is compared to the point-wise estimate of total correlation described in the text.}\label{fig:finance}
\vspace{-4mm}
\end{figure}

CorEx naturally produces sparse graphs because a connection with a new latent factor is formed only if it contains unique information.
While the thresholded graph in Fig.~\ref{fig:spgraph} is tree-like, the full hierarchical structure is not, as shown in Fig.~\ref{fig:big}. The stock with the largest overlap in two groups was TGT, or Target, which was strongly connected to a group containing department stores like Macy's and Nordstrom's and was also strongly connected to a group containing home improvement retailers Lowe's, Home Depot,  and Bed, Bath, and Beyond. 
The next two stocks with large overlaps in two groups were Conoco-Phillips and Marathon Oil Corp.\ which were both connected to a group containing oil companies and another group containing property-related businesses. 

\section{Conclusions}\label{sec:conclusion}

We have demonstrated a method for constructing hierarchical representations that are maximally informative about the data. Each latent factor and layer contributes to tighter bounds on the information in the data in a quantifiable way.  
The optimization we presented to construct these representations has linear computational complexity and constant sample complexity which makes it attractive for real-world, high-dimensional data.
Previous results on the special case of tree-like representations outperformed state-of-the-art methods on synthetic data and demonstrated promising results for unsupervised learning on diverse data from human behavior, biology, and language~\cite{nips2014}. 
By introducing this theoretical foundation for hierarchical decomposition of information, we were able to extend previous results to enable discovery of overlapping structure in data and to provide bounds on the information contained in data. 
 
 Specifying the number and cardinality of latent factors to use in a representation is an inconvenience shared with other deep learning approaches. Unlike other approaches, the bounds in Sec.~\ref{sec:bounds} quantify the trade-off between representation size and tightness of bounds on information in the data. Methods to automatically size representations to optimize this trade-off will be explored in future work. 
 Other intriguing directions include using the bounds presented to characterize RBMs and auto-encoders~\cite{bengioreview}, and exploring connections to the information bottleneck~\cite{slonimmvib,slonimthesis}, multivariate information measures~\cite{williamsbeer,griffith,steudelay}, EM~\cite{dempster,EM_max}, and ``recognition models''~\cite{recognition}.
 
 The combination of a domain-agnostic theoretical foundation with rigorous, information-theoretic guarantees suggests compelling applications in domains with complex, heterogeneous, and highly correlated data such as gene expression and neuroscience~\cite{isbi_blood}. Preliminary experiments have produced intriguing results in these domains and will appear in future work.

\subsubsection*{Acknowledgments}
This research was supported in part by AFOSR grant FA9550-12-1-0417 and DARPA grant W911NF-12-1-0034.

\bibliographystyle{unsrt}
\bibliography{gversteeg,galstyan,gversteeg_t,versteeg} 

\appendix

\clearpage
\counterwithin{figure}{section}

\section*{Supplementary Material for ``Maximally Informative Hierarchical Representations of High-Dimensional Data''}

\section{Proof of Theorem~\ref{tclb}}\label{sec:lower}
\begin{theorem*}
\emph{Basic Decomposition of Information}

If $Y$ is a representation of $X$ and we define,
\begin{align*} 
TC_{L} (X; Y) &\equiv  \sum_{i=1}^n I(Y:X_i) - \sum_{j=1}^m I(Y_j:X),
\end{align*}
then the following bound and decomposition holds. 
\begin{align*} 
TC(X) \geq TC(X;Y) = TC(Y) + TC_L(X;Y)
\end{align*}
\end{theorem*}
\begin{proof}
The first inequality trivially follows from Eq.~\ref{eq:tcxy} since we subtract a non-negative quantity (a KL divergence) from $TC(X)$. For the second equality, we begin by using the definition of $TC(X;Y)$, expanding the entropies in terms of their definitions as expectation values.  We will use the symmetry of mutual information, $I(A:B) = I(B:A)$, and the identity $I(A:B) = \mathbb E_{A,B} \log (p(a|b)/p(a))$. By definition, the full joint probability distribution can be written as $p(x,y) = p(y|x) p(x) = \prod_j p(y_j|x) p(x)$. 
\begin{align}\label{eq:identity}
I(X:Y) &=  \mathbb E_{X,Y} \left[\log \frac{p(y|x)}{p(y)}  \right]   \notag  \\
&= \mathbb E_{X,Y} \left[ \log  \frac{\prod_{j=1}^m p(y_j)}{p(y)} \frac{\prod_{j=1}^m p(y_j|x)}{\prod_{j=1}^m p(y_j)} \right] \notag \\
&= - TC(Y) + \sum_{j=1}^m I(Y_j : X)
\end{align}
Replacing $I(X:Y)$ in Eq.~\ref{eq:tcxy} completes the proof. 
\end{proof}

\section{Proof of Theorem~\ref{tcub}}\label{sec:upper}
\begin{theorem*}
\emph{Upper Bounds on $TC(X)$}

If $Y^{1:r}$ is a hierarchical representation of $X$ and we define $Y^0 \equiv X$, and additionally $m_r=1$ and all variables are discrete, then, 
\benn 
TC(X) &\leq  TC(Y^1) + TC_L(X;Y^1) + \sum_{i=1}^n H(X_i|Y^1)\\
TC(X) &\leq  \sum_{k=1}^r \left( TC_L(Y^{k-1};Y^k) + \sum_{i=1}^{m_{k-1}} H(Y^{k-1}_i|Y^k) \right).
\eenn
\end{theorem*}

\begin{proof}
We begin by re-writing Eq.~\ref{eq:basic} as $TC(X) = TC(X|Y^1)  + TC(Y^1) + TC_L(X;Y^1)$. 
Next, for discrete variables, $TC(X|Y^1) \leq \sum_i H(X_i|Y)$, giving the inequality in the first line.
The next inequality follows from iteratively applying the first inequality as in the proof of Thm.~\ref{hierarchy}. Because $m_r=1$, we have $TC(Y^r) =0$.
\end{proof}

\section{Derivation of Eqs.~\ref{eq:label} and \ref{eq:marg}}\label{sec:derive}

We want to optimize the objective in Eq.~\ref{eq:AI}.
\be
\max_{p(y|x)} & \sum_{i=1}^n \alpha_i I(Y:X_i) - I(Y:X) \\
\textrm{s.t.} & \sum_{y} p(y|x) = 1
\ee
For simplicity, we consider only a single $Y_j$ and drop the $j$ index. Here we explicitly include the condition that the conditional probability distribution for $Y$ should be normalized. We consider $\alpha$ to be a fixed constant in what follows. 

We proceed using Lagrangian optimization.
We introduce a Lagrange multiplier $\lambda(x)$ for each value of $x$ to enforce the normalization constraint and then reduce the constrained optimization problem to the unconstrained optimization of the objective $\mathcal L$.  
\benn
\mathcal L = \sum_{\bar x,\bar y} p(\bar x) p(\bar y|\bar x)   \big(\sum_i \alpha_i (\log p(\bar y|\bar x_i)-\log p(\bar y))
\\  - (\log p(\bar y|\bar x) - \log p(\bar y))   \big)  
\\  + \sum_{\bar x} \lambda(\bar x) (\sum_{\bar y} p(\bar y| \bar x) -1)
\eenn
Note that we are optimizing over $p(y|x)$ and so the marginals $p(y|x_i),p(y)$ are actually linear functions of $p(y|x)$.  
Next we take the functional derivatives with respect to $p(y|x)$ and set them equal to 0. We re-use a few identities. Unfortunately, $\delta$ on the left indicates a functional derivative while on the right it indicates a Kronecker delta. 
\benn
\frac{\delta p(\bar y | \bar x)}{\delta p(y|x)} &= \delta_{y,\bar y} \delta_{x,\bar x} \\
\frac{\delta p(\bar y)}{\delta p(y|x)} &= \delta_{y,\bar y} p(x) \\ 
\frac{\delta p(\bar y| \bar x_i)}{\delta p(y|x)} &= \delta_{y,\bar y} \delta_{x_i,\bar x_i} p(x)/p(x_i)
\eenn
Taking the derivative and using these identities, we obtain the following.
\benn
\frac{\delta \mathcal L}{\delta p(y|x)} &= \lambda(x) + \\
&p(x) \log \frac{\prod_i (p(y|x_i)/p(y))^{\alpha_i}}{p(y|x)/p(y)} + \\
& \sum_{\bar x,\bar y} p(\bar x) p(\bar y|\bar x)   \big(\sum_i \alpha_i (\frac{\delta_{y,\bar y} \delta_{x_i,\bar x_i} p(x)}{p(x_i)  p(\bar y|\bar x_i)} \\
& -\delta_{y,\bar y} p(x)/ p(\bar y))
\\ & - (\frac{\delta_{y,\bar y} \delta_{x,\bar x} }{ p(\bar y|\bar x)} - \delta_{y,\bar y} p(x) / p(\bar y))   \big) 
\eenn
Performing the sums over $\bar x, \bar y$ leads to cancellation of the last three lines. Then we set the remaining quantity equal to zero. 
\benn
\frac{\delta \mathcal L}{\delta p(y|x)} &= \lambda(x) + 
p(x) \log \frac{\prod_i p(y|x_i)/p(y)}{p(y|x)/p(y)} = 0
\eenn
This leads to the following condition in which we have absorbed constants like $\lambda (x)$ in to the partition function, $Z(x)$.
\benn
p(y|x) &= \frac1{Z(x)} p(y) \prod_{i=1}^n \left(\frac{p(y|x_i)}{p(y)}\right)^{\alpha_{i}}
\eenn
We recall that this is only a formal solution since the marginals themselves are defined in terms of $p(y|x)$.
\benn
p(y) &= \sum_{\bar x} p(\bar x) p(y|\bar x) \\
 p(y|x_i) &= \sum_{\bar x} p(y|\bar x) p(\bar x) \delta_{x_i, \bar x_i} /p(x_i)
\eenn
If we have a sum over independent objectives like Eq.~\ref{eq:fullopt} for $j=1,\ldots,m$, we just place subscripts appropriately. 
The partition constant, $Z_j(x)$ can be easily calculated by summing over just $|Y_j|$ terms. 

\section{Updates Do Not Decrease the Objective}\label{sec:Z}
The detailed proof of this largely follows the convergence proof for the iterative updating of the information bottleneck~\cite{tishby}.
\begin{theorem}
Assuming $\alpha_1,\ldots,\alpha_n  \in [0,1]$, iterating over the update equations given by Eq.~\ref{eq:update} and Eq.~\ref{eq:marg} never decreases the value of the objective in Eq.~\ref{eq:AI} and is guaranteed to converge to a stationary fixed point.
\end{theorem}
\begin{proof}
 First, we define a functional of the objective with the marginals considered as separate arguments.
\benn
 \lefteqn{\mathcal F[p(x_i|y), p(y), p(y|x)] \equiv} \\
&& \sum_{x,y} p(x) p(y|x) \left( \sum_i \alpha_i \log \frac{p(x_i|y)}{p(x_i)} - \log \frac{p(y|x)}{p(y)}\right) 
 \eenn
As long as $\alpha_i \leq 1$, this objective is upper bounded by $TC_L(X;Y)$ and Thm.~\ref{hierarchy} therefore guarantees that the objective is upper bounded by the constant $TC(X)$.
 Next, we show that optimizing over each argument separately leads to the update equations given. 
 We skip re-calculation of terms appearing in Sec.~\ref{sec:derive}.  Keep in mind that for each of these separate optimization problems, we should introduce a Lagrange multiplier to ensure normalization.
{\allowdisplaybreaks
\begin{align*}
  \frac{\delta \mathcal F}{\delta p(y)} &= \lambda +  \sum_{\bar x} p(y|\bar x) p(\bar x) / p(y)  \\
\frac{\delta \mathcal F}{\delta p(x_i|y)} &=  \lambda_i + \sum_{\bar x} p(y|\bar{x}) p(\bar x) \alpha_i \delta_{\bar x_i, x_i}/p(x_i|y) \\ 
\frac{\delta \mathcal F}{\delta p(y|x)} &= \lambda(x) + p(x) \left( \sum_i \alpha_i \log \frac{p(x_i|y)}{p(x_i)} - \log \frac{p(y|x)}{p(y)} - 1\right)
\end{align*}
}
 Setting each of these equations equal to zero recovers the corresponding update equation. 
 Therefore, each update corresponds to finding a local optimum. 
 Next, note that the objective is (separately) concave in both $p(x_i|y)$ and $p(y)$, because $\log$ is concave. Furthermore, the terms including $p(y|x)$ correspond to the entropy $H(Y|X)$, which is concave. 
Therefore each update is guaranteed to increase the value of the objective (or leave it unchanged). Because the objective is upper bounded this process must converge (though only to a local optimum, not necessarily the global one). 
\end{proof}

\section{Convergence for S\&P 500 Data}\label{sec:restart}

Fig.~\ref{fig:convergence} shows the convergence of the lower bound on $TC(X)$ as we step through the iterative procedure in Sec.~\ref{sec:update} to learn a representation for the finance data in Sec.~\ref{sec:experiments}. As in the synthetic example in Fig.~\ref{fig:multiple}(a), convergence occurs quickly. The iterative procedure starts with a random initial state. Fig.~\ref{fig:convergence} compares the convergence for 10 different random initializations. In practice, we can always use multiple restarts and pick the solution that gives the best lower bound.

\begin{figure}[htbp] 
   \centering
  \includegraphics[width=0.93\columnwidth]{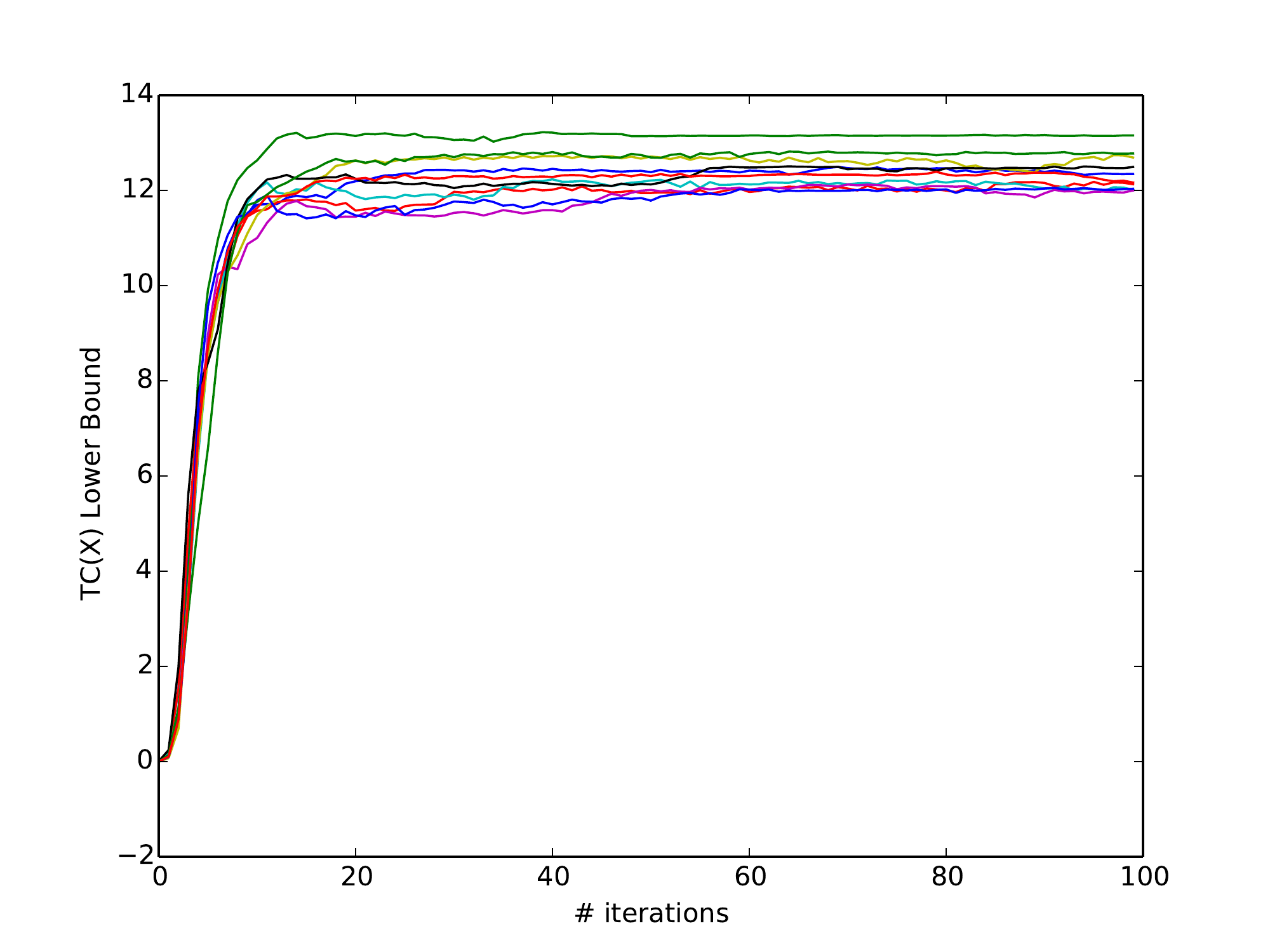} 
   \caption{Convergence of the lower bound on $TC(X)$ as we perform our iterative solution procedure, using multiple random initializations. }
   \label{fig:convergence}
\end{figure}

%

\clearpage
\begin{sidewaysfigure}[ht]
\centering
    \includegraphics[width=9in]{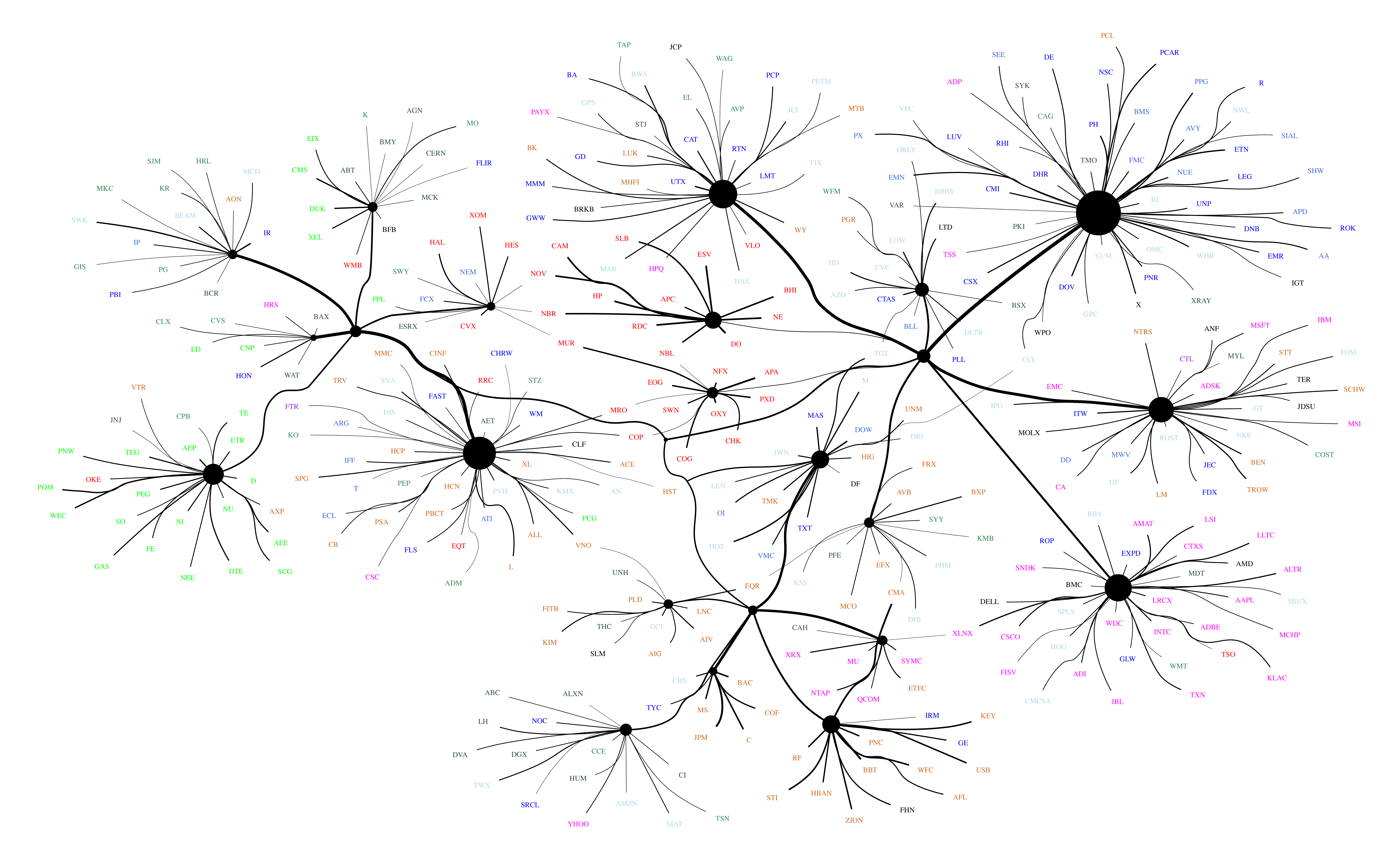} \\
       \includegraphics[width=8in]{Figs/key3.pdf}
    \caption{A larger version of the graph in Fig.~\ref{fig:spgraph} with a lower threshold on for displaying edge weights (color online). }
    \vspace{2in}
    \label{fig:big}
\end{sidewaysfigure}

\clearpage
\begin{sidewaysfigure}[ht]
\centering
    \includegraphics[width=6in]{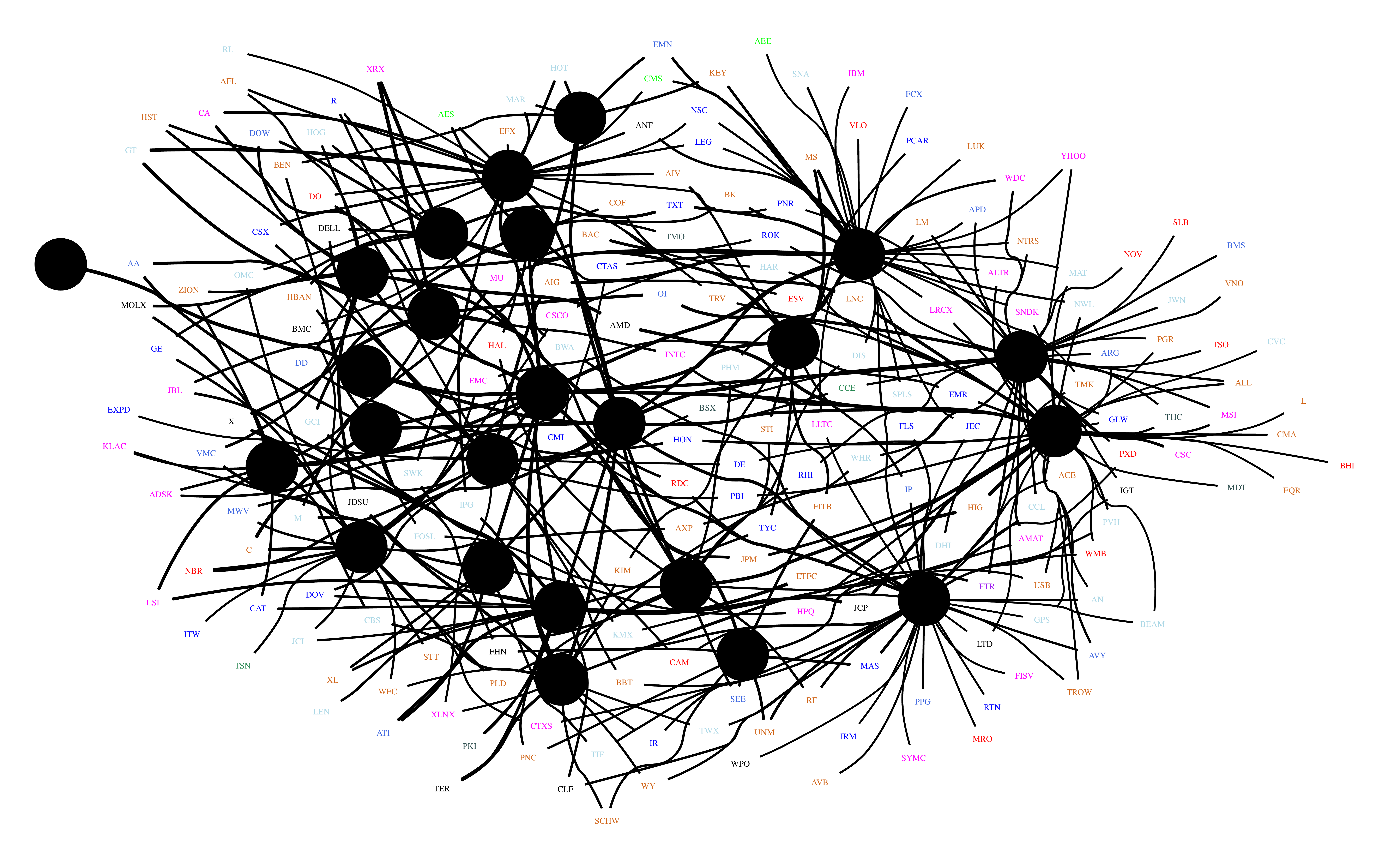}  \includegraphics[width=3in]{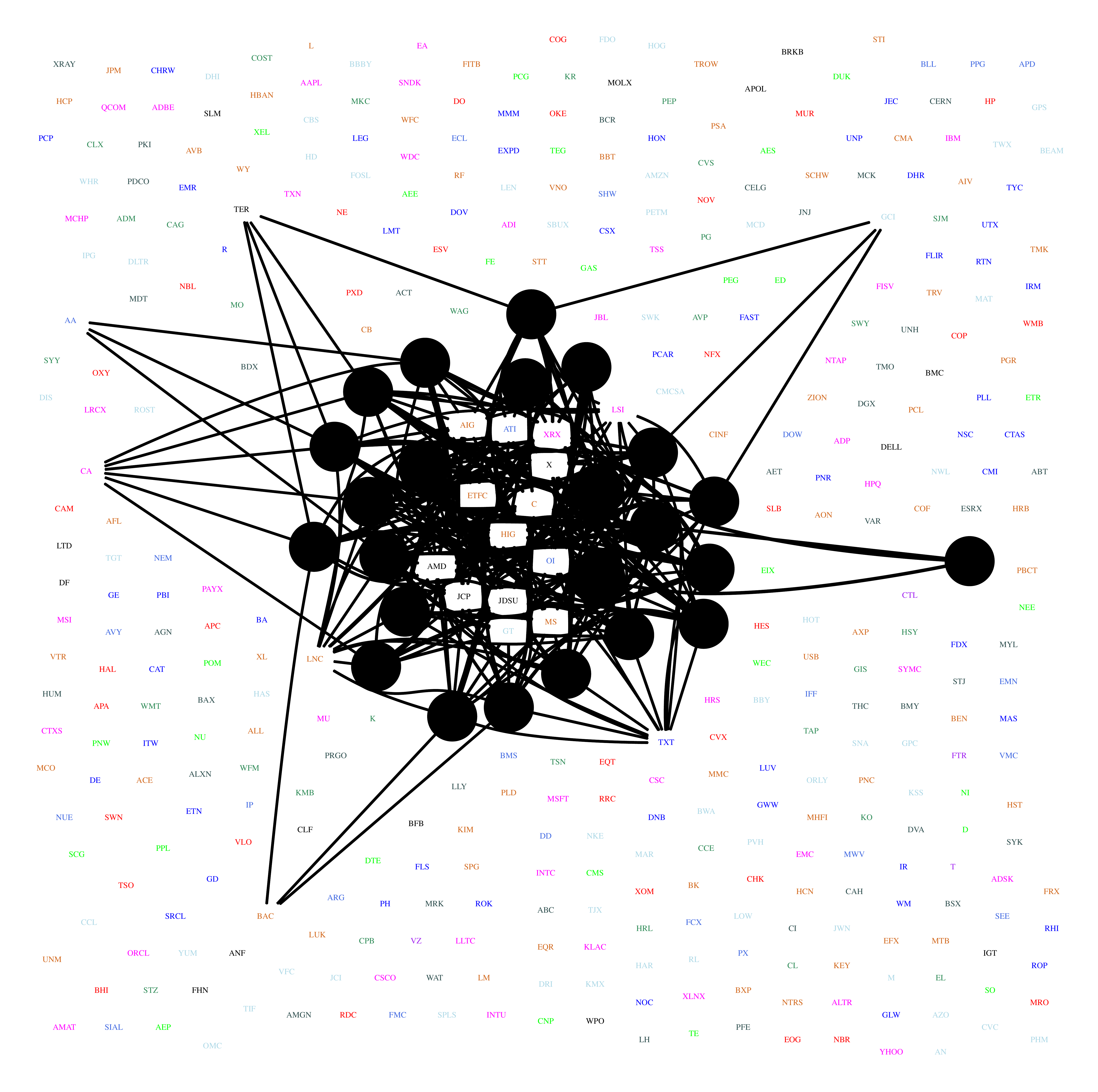} \\
       \includegraphics[width=8in]{Figs/key3.pdf}
    \caption{For comparison, we constructed a structure similar to Fig.~\ref{fig:big} using restricted Boltzmann machines with the same number of layers and hidden units. On the right, we thresholded the (magnitude) of the edges between units to get the same number of edges (about 400). On the left, for each unit we kept the two connections with nodes in higher layers that had the highest magnitude and restricted nodes to have no more than 50 connections with lower layers (color online).}
    \label{fig:nnet}
    \vspace{2in}
\end{sidewaysfigure}

\end{document}